\documentclass[11pt,a4paper]{article}

\usepackage[comma]{natbib}
\usepackage{color}
\usepackage{graphicx}
\usepackage{pdfsync}

\usepackage{amsmath,amssymb,amsthm,amsopn}
\usepackage[ruled]{algorithm2e}

\usepackage[mathscr]{euscript}

\usepackage{pgfplots}

\pgfplotsset{compat=1.8}

\newtheorem{lemma}{Lemma}
\newtheorem{theorem}{Theorem}
\theoremstyle{remark}

\renewcommand{\today}{\begingroup
\number \day\space  \ifcase \month \or January\or February\or
March\or April\or May\or June\or July\or August\or September\or
October\or November\or December\fi \space  \number \year \endgroup}

\theoremstyle{plain}

\newtheorem{teor*}{Teorema}

\theoremstyle{definition}

\title{Minimum adjusted Rand index for two clusterings\\ of a given size}

\author{Jos\'e E. Chac\'on\footnote{Departamento de
Matem\'aticas, Universidad de Extremadura, E-06006 Badajoz, Spain. E-mail:
{\tt jechacon@unex.es}} \ and Ana I. Rastrojo\footnote{Departamento de Matem\'aticas, IES Sierra La Calera, E-06150 Santa Marta de los Barros, Badajoz, Spain. E-mail: {\tt anabel.rastrojo@gmail.com}}}

\begin{document}

\maketitle

\begin{abstract}
\noindent The adjusted Rand index (ARI) is commonly used in cluster analysis to measure the degree of agreement between two data partitions. Since its introduction, exploring the situations of extreme agreement and disagreement under different circumstances has been a subject of interest, in order to achieve a better understanding of this index. Here, an explicit formula for the lowest possible value of the ARI for two clusterings of given sizes is shown, and moreover a specific pair of clusterings achieving such a bound is provided.
\end{abstract}


\newpage

\section{Introduction}

The adjusted Rand index is one of the most commonly used similarity measures to compare two clusterings of a given set of objects. Indeed, it is the recommended criterion for external clustering evaluation in the seminal study of \cite{MC86}. Nevertheless, many other measures for external clustering evaluation were recently surveyed in \cite{M16}.

Initially, \cite{R71} considered a similarity index between two clusterings (the Rand index) defined as the proportion of object pairs that are either assigned to the same cluster in both clusterings or to different clusters in both clusterings. However, \cite{MA84} noted that such an index does not take into account the possible agreement by chance, and \cite{HA85} introduced a corrected-for-chance version of the Rand index, which is usually known as the adjusted Rand index (ARI).

Exploring the situations of extreme agreement, as measured by the ARI, has been a subject of interest since the very inception of this index. Indeed, \cite{HA85} posed the problem of finding the maximum ARI subject to given clustering marginals; i.e., when constrained to have fixed, given cluster sizes in each of the clusterings. Numerical algorithms to tackle this problem were developed initially by \cite{M92}, and later by \cite{BS08} and \cite{SHB15}, and an explicit solution for clusterings of size 2 has been recently shown in \cite{Ch20}.

A related but different problem concerns the obtention of lower bounds for the ARI of two clusterings of given sizes. According to \citet[][p. 631]{M16}, ``the lower bound is usually hard to calculate''. It should be noted that, due to the correction for chance, the ARI may take negative values for extremely discordant clusterings. This happens when the agreement between the two clusterings is less than the expected agreement when the clusters assignments are made at random, keeping the given marginals. Hence, finding the minimum possible ARI value allows quantifying how extreme is the discordance between two clusterings of given sizes.


Moreover, if the interest is to measure discordance instead of agreement, the ARI can be transformed into a semimetric by considering ${\rm ARD}=1-{\rm ARI}$ \citep{Ch19}. Thus, perfect agreement corresponds to null discordance, or ${\rm ARD}=0$, and the case of less agreement than random assignment is related to values of ${\rm ARD}>1$. However, in general, when dealing with semimetrics for measuring clustering disagreement it is useful to normalize them so that they take values in $[0,1]$; for instance, \cite{CDGH06} and \cite{M16} explored such a normalization for different distances between partitions. Therefore, obtaining the minimum ARI value makes it possible to define a normalized version of the ARD.

The main contribution of this paper is to find a lower bound for the ARI of two clusterings of given sizes, and to show that this bound is indeed the best possible one, since it is attained by an explicit pair of clusterings. More precise notation is introduced in Section \ref{sec:not}, where in addition the main result is rigourously stated. Two numerical examples showing the possible applications of this result are presented in Section \ref{sec:ex}, and the proofs of the main result and another auxiliar lemma of independent interest are given in Section \ref{sec:2}.

\section{Notation and main result}\label{sec:not}

A clustering of a set $\mathcal X$ of $n$ objects is a partition of $\mathcal X$ into non-empty, disjoint and exhaustive classes, called clusters. The number of such classes is known as the size of the clustering. Given two clusterings $\mathscr C=\{C_1,\dots,C_r\}$ and $\mathscr D=\{D_1,\dots,D_s\}$, of sizes $r$ and $s$, respectively, all the information regarding their concordance is registered in the $r\times s$ matrix ${\mathbf N}$ whose $(i,j)$th element $n_{ij}$ records the cardinality of $C_i\cap D_j$. This matrix is usually known as confusion matrix or contingency table. Its row-wise and column-wise totals, $(n_{1+},\dots,n_{r+})$ and $(n_{+1},\dots,n_{+s})$, with $n_{i+}=\sum_{j=1}^sn_{ij}$ and $n_{+j}=\sum_{i=1}^rn_{ij}$, give an account of the cluster sizes in $\mathscr C$ and $\mathscr D$, respectively, and are commonly referred to as the marginals, or marginal clustering distributions. Note that all cluster sizes must be strictly greater than zero in order to respect the assumptions on the clustering sizes.

The Rand index is a summary statistic for ${\mathbf N}$, based on inspecting the behaviour of object pairs across the two clusterings. There are four possible types of object pairs, formed by taking into account if: a) both objects belong to the same cluster in both clusterings, b) they belong to the same cluster in $\mathscr C$ but to different clusters in $\mathscr D$, c) they belong to different clusters in $\mathscr C$ but to the same cluster in $\mathscr D$, and d) they belong to different clusters in both clusterings. The cardinalities of each of these categories will be denoted $a$, $b$, $c$ and $d$, respectively. They can be easily expressed in terms of the entries of ${\mathbf N}$ and its marginals; for instance, \cite{HA85} noted that
\begin{align*}
a&=\frac{\big(\sum_{i=1}^r\sum_{j=1}^sn_{ij}^2\big)-n}2,\\
b&=\frac{\sum_{i=1}^rn_{i+}^2-\sum_{i=1}^r\sum_{j=1}^sn_{ij}^2}2,\\
c&=\frac{\sum_{j=1}^sn_{+j}^2-\sum_{i=1}^r\sum_{j=1}^sn_{ij}^2}2,\\
d&=\frac{\big(\sum_{i=1}^r\sum_{j=1}^sn_{ij}^2\big)+n^2-\sum_{i=1}^rn_{i+}^2-\sum_{j=1}^sn_{+j}^2}2.
\end{align*}

With this notation, the Rand index is obtained as ${\rm RI}=(a+d)/(a+b+c+d)=(a+d)/N$ where $N=a+b+c+d={n\choose 2}=n(n-1)/2$ is the total number of pairs of objects from $\mathcal X$. It takes values in $[0,1]$, with 1 corresponding to perfect agreement between the clusterings and 0 attained for the comparison of the two so-called trivial clusterings: one with all the $n$ objects in a single cluster, and the other one with $n$ clusters with a single object in each of them \citep[see][]{ANM06}.

One of the drawbacks of the Rand index is that it does not take into account the possibility of agreement by chance between the two clusterings \citep{MA84}. Hence, \cite{HA85} obtained $\mathbb E[{\rm RI}]$, the expected value of this index when the partitions are made at random, but keeping the same marginal clustering distributions, and suggested to alternatively use the ARI, a corrected-for-chance version of the Rand index defined by ${\rm ARI}=({\rm RI}-\mathbb E[{\rm RI}])/(1-\mathbb E[{\rm RI}])$. \cite{St04} provided a concise formula for the ARI, which reads as follows:
$${\rm ARI}=\frac{N(a+d)-\{(a+b)(a+c)+(c+d)(b+d)\}}{N^2-\{(a+b)(a+c)+(c+d)(b+d)\}}.$$
Note that the ARI is undefined if $r=s=1$, so it will be assumed henceforth that at least one of the clusterings has more than one cluster, i.e., that $\max\{r,s\}>1$.

These preliminaries allow us to formulate the main result of this paper, whose proof is deferred to Section \ref{sec:2}.

\begin{theorem}\label{thm1}
The minimum ARI for two clusterings of an arbitrary number of objects, with given sizes $r$ and $s$, respectively, is attained for a comparison of precisely $n=r+s-1$ objects, in which the $r\times s$ contingency table ${\mathbf N}$ has exactly one row of ones, exactly one column of ones and all the remaining entries are zeroes. Such a minimum value can be explicitly written as
\begin{equation}\label{eq:minARI}
\min{\rm ARI}=\left[1-\frac12{r+s-1\choose 2}\left\{{r\choose 2}^{-1}+{s\choose 2}^{-1}\right\}\right]^{-1}
\end{equation}
if $\min\{r,s\}\geq2$ and $\min{\rm ARI}=0$ if $\min\{r,s\}=1$.
\end{theorem}

The expression for the minimum ARI given in Theorem \ref{thm1} is equivalent to, but notably simpler than, the one previously announced in \cite{Ch19}.

If in addition the clustering sizes are allowed to vary, then it is easily seen that the minimum possible value of the ARI is $-1/2$, which corresponds to a $2\times2$ matrix with one entry equal to zero and all the remaining entries equal to one.

Furthermore, for $r=s\geq2$, Equation (\ref{eq:minARI}) simplifies to $-r/(3r-2)$, so it follows that for $r=s$ the range of possible ARI values approaches $[-1/3,1]$  as $r$ increases. 
Moreover, in order to get insight on the behaviour of the minimum ARI for large values of $r$ and $s$ it is useful to note that, by means of the simple first order approximation ${r\choose 2}\sim r^2/2$, it is possible to express
$$\min{\rm ARI}\approx-\frac{2r^2s^2}{r^4+2r^3s+2rs^3+s^4}$$
as $r$ and $s$ increase.

\section{Examples}\label{sec:ex}

\subsection{A synthetic data example}

Theorem \ref{thm1} is useful to appreciate how extreme is the discordance between two distant clusterings of given sizes.

For instance, let us consider the example presented in Table 3 in \citet{St04}, which concerns the comparison of two partitions of $n=13$ objects into $r=s=5$ clusters. The $5\times 5$ confusion matrix for this example is given by
$$\begin{pmatrix}
1 & 0 & 1 & 1 & 0\\
0 & 1 & 0 & 0 & 1\\
1 & 0 & 1 & 0 & 1\\
0 & 1 & 0 & 1 & 0\\
1 & 0 & 1 & 0 & 1
\end{pmatrix}.$$

\cite{Ch19} noted that this example deals with two very distant clusterings. More precisely, it is easy to check that for this comparison we have $a=0$, $b=c=11$ and $d=56$, so that ${\rm ARI}=-242/1474\simeq-0.164$. The fact that ${\rm ARI}<0$ already indicates that the agreement between these two partitions is less than the expected agreement if the label assignments would have made at random, so that supports the idea that the two clusterings are quite distant.

But one may wonder if two partitions with 5 clusters each can be made much more distant that these two, and that is precisely the question that Theorem \ref{thm1} solves, since it shows that the minimum possible agreement for $r=s=5$ is $\min{\rm ARI}=-5/13\simeq-0.385$. Thus, for two clusterings of size 5 the range of possible ARI values is $[-0.385,1]$, so the value $-0.164$ for the partitions in this example is indeed quite close to the lower limit.

Moreover, Theorem \ref{thm1} also shows that the lowest possible value of the ARI for two clusterings of size $5$ is attained for the comparison of two clusterings of $9$ objects whose confusion matrix is
$$\begin{pmatrix}
1 & 1 & 1 & 1 & 1\\
1 & 0 & 0 & 0 & 0\\
1 & 0 & 0 & 0 & 0\\
1 & 0 & 0 & 0 & 0\\
1 & 0 & 0 & 0 & 0
\end{pmatrix},$$
or any other that can be obtained by permuting the rows and/or the columns of the former.

\subsection{A real data example}

While clustering comparisons are often made based on indices, some authors advocate the advantages of using distances as dissimilarity measures \cite[see][p. 620]{M16}. Hence, as noted in the Introduction, another application of Theorem \ref{thm1} is that it allows normalizing the dissimilarity measure ${\rm ARD}=1-{\rm ARI}$ so that it takes values in $[0,1]$, which makes it easier to appreciate the relative closeness of two partitions with respect to a third one.

In this sense, let us consider the yeast data set introduced in \cite{NK91,NK92}, a version of which is publicly available at the UCI machine learning repository ({\tt https://archive.ics.uci.edu/}). The data consist of 8 variables measured on $n=1484$ proteins. An additional label variable is available, that classifies these proteins according to their cellular localization sites as CYT or ME3, which induces a partition that can be thus viewed as the ground truth. Then, the goal is to compare the partitions obtained by different clustering procedures against the true classification.

Gaussian mixture models (GMMs) and shifted asymmetric Laplace (SAL) mixture models were applied in \cite{FBM14} to cluster this data set. The reported fitted SAL mixture model has 2 clusters, with ${\rm ARI}=0.81$, whereas the fitted GMM has 3 clusters and ${\rm ARI}=0.56$. The higher value of the ARI already seems to indicate a better fit for the SAL mixture model but, in order to better appreciate the relative gains of this method over the GMM, it is useful to calculate the normalized ARD, which takes values in $[0,1]$.

By using Theorem \ref{thm1}, the normalized ARDs for the SAL mixture model fit and the GMM fit can be computed to be $0.13$ and $0.33$, respectively, thus showing on a $[0,1]$ scale how the SAL mixture model fit is quite closer to the true classification than the GMM fit. Moreover, after aggregating the results for 25 model fits with a fixed number of 2 clusters, based on random initializations with 70 percent of the true labels taken as known, the reported results entail that the normalized ARDs for the SAL and GMM clusterings against the ground truth were 0.09 and 0.72, respectively, thus showing a considerably lower normalized dissimilarity for the SAL mixture model clustering against the GMM partition.

\section{Proofs}\label{sec:2}

The proof of Theorem \ref{thm1} makes use of the following result, which is of independent interest. Intuitively, it shows that if a certain amount is to be distributed among several parts, the configuration that yields the maximum sum of the part squares is that which accumulates the highest possible quantity in one of the parts and keeps the remaining ones to their minimum.

\begin{lemma}\label{lem1}
Let $a_1\geq a_2\geq\cdots\geq a_p$ and $t\geq\sum_{i=1}^pa_i$ be real numbers and consider the region
$$\mathcal A\equiv \mathcal A(t;a_1,\dots,a_p)=\big\{(x_1,\dots,x_p)\in\mathbb R^p\colon \textstyle\sum_{i=1}^px_i=t\text{ and }x_i\geq a_i\text{ for all }i=1,\dots,p\big\}.$$
The maximum of $\sum_{i=1}^px_i^2$ over $\mathcal A$ is attained for $x_1=t-\sum_{i=2}^pa_i,\, x_2=a_2,\dots,x_p=a_p$. Hence, $\max\big\{\sum_{i=1}^px_i^2\colon(x_1,\dots,x_p)\in\mathcal A\big\}=\big(t-\sum_{i=2}^pa_i\big)^2+\sum_{i=2}^pa_i^2$.
\end{lemma}

\begin{proof}
The result follows by noting that if $a\leq b$ then $a^2+b^2\leq(a-c)^2+(b+c)^2$ for any $c\geq0$.
\end{proof}


Now we are ready to prove the main result of the paper.

\begin{proof}[Proof of Theorem \ref{thm1}]
First note that minimizing the ARI is equivalent to maximizing the semimetric ${\rm ARD}=1-{\rm ARI}$ introduced in \cite{Ch19}, where it is also shown that it can be readily expressed as
\begin{equation}\label{eq:ard}
{\rm ARD}\equiv{\rm ARD}(a,b,c,d)=\frac{N(b+c)}{(a+b)(b+d)+(a+c)(c+d)}.
\end{equation}
It is clear that the roles of $b$ and $c$ in (\ref{eq:ard}) are interchangeable, in the sense that ${\rm ARD}(a,b,c,d)={\rm ARD}(a,c,b,d)$. The same is true for the roles of $a$ and $d$. Moreover, ${\rm ARD}$ is clearly a decreasing function of $a$ and $d$, so its maximum value is attained for the lowest possible values of $a$ and $d$. 

\cite{ANM06} noted that $a=0$ if and only if $n_{ij}\in\{0,1\}$ for all $i=1,\dots,r$ and $j=1,\dots,s$ and $d=0$ if and only if $\min\{r,s\}=1$. Hence, when one of the clusterings consists of a single cluster, the contingency table with maximum ARD is a row or column vector of ones, with resulting ${\rm ARD}=1$, so minimum ${\rm ARI}=0$.

On the other hand, if $\min\{r,s\}\geq2$ then necessarily $d>0$, but it is equally possible to have $a=0$ if all the entries of ${\mathbf N}$ are just zeroes or ones, so this will be imposed henceforth. Notice that this yields $n\leq rs$, which means that the highest values of the ARD are achieved when the number of objects is small. For $a=0$ we have $d=N-(b+c)$, and the ARD simplifies to
\begin{equation}\label{eq:ard2}
{\rm ARD}=\frac{N(b+c)}{b^2+c^2+\{N-(b+c)\}(b+c)}=\frac{N}{N-2bc/(b+c)},
\end{equation}
which is an increasing function of $b$ and $c$. So, to maximize it, we must find the maximum possible values for $b$ and $c$.

Since $a=0$, it follows that  $b=\big(\sum_{i=1}^rn_{i+}^2-n\big)/2$ and $c=\big(\sum_{j=1}^sn_{+j}^2-n\big)/2$. Hence, maximizing $b$ is equivalent to maximizing the sum of the squared sizes of the clusters of $\mathscr C$, constrained to the facts that the total size is $n$ and each cluster has size greater than or equal to one (because degenerate, empty clusters are not allowed). This is exactly the setting of Lemma \ref{lem1} for $p=r$, $a_1=\dots=a_r=1$ and $t=n$. So for $n\geq r$ (which is necessary to have $r$ non-empty clusters in $\mathscr C$), the maximum value of $b$ is attained when there is a cluster in $\mathscr C$ with $n-(r-1)$ objects and the remaining $r-1$ clusters have one object each, so that $\sum_{i=1}^rn_{i+}^2=\{n-(r-1)\}^2+r-1$.

Moreover, the fact that all $n_{ij}\in\{0,1\}$ also implies that the maximum size of any cluster in $\mathscr C$ is $s$, which for the configuration maximizing $b$ yields $n-(r-1)\leq s$. And, in view of the maximum value of $\sum_{i=1}^rn_{i+}^2$, among all the sample sizes $n$ that satisfy the latter constraint, the one for which $b$ is maximum corresponds precisely to $n-(r-1)=s$, that is, to $n=r+s-1$. Hence, the confusion matrix that maximizes $b$ must have one row with all its entries equal to one, and each of the remaining rows having exactly one entry equal to one and all the rest equal to zero. In principle, the nonzero entries of the latter rows could be arbitrarily placed but, mimicking the above reasoning regarding $b$, the value of $c$ is maximized when there is a column with all its entries equal to one, so the contingency table configuration that maximizes the ARD must be precisely the one announced in the statement of the theorem.

In addition, it is straightforward to check that the configuration that maximizes the ARD has $a=0$, $b={s\choose 2}$, $c={r\choose 2}$ and $d={n\choose 2}-{r\choose 2}-{s\choose 2}=(r-1)(s-1)$ since $n=r+s-1$. Hence, from (\ref{eq:ard2}) it follows that the maximum ARD is given by
$$\left[1-2{r+s-1\choose 2}^{-1}{r\choose 2}{s\choose 2}\bigg/\left\{{r\choose 2}+{s\choose 2}\right\}\right]^{-1}$$
so that the minimum ARI is as stated in the theorem.
\end{proof}

\bigskip

\noindent{\bf Acknowledgments.}  The first author acknowledges the support of the Spanish Ministerio de Econom\'\i a y Competitividad grant PID2019-109387GB-I00 and the Junta de Extremadura grant GR18016.

\bibliographystyle{apalike}

\end{document}